\documentclass[letterpaper, 10 pt, conference]{ieeeconf}  

\IEEEoverridecommandlockouts                              

\usepackage{amsthm}

\usepackage{microtype}
\usepackage{graphicx}
\usepackage{subfigure}
\usepackage{booktabs} 
\usepackage{algorithm}
\usepackage{algorithmic}
\usepackage[hidelinks]{hyperref}
\usepackage{url}
\usepackage{dsfont} 
\usepackage{siunitx}
\usepackage{amsmath,amssymb,amsfonts}
\usepackage{cleveref}
\usepackage{mathrsfs}
\usepackage{hyperref}

\newtheorem{assumption}{Assumption}
\newtheorem{lemma}{Lemma}
\newtheorem{theorem}{Theorem}

\title{\LARGE \bf
A Zeroth-Order Momentum Method  for Risk-Averse \\ Online Convex Games
}

\author{Zifan Wang, Yi Shen, Zachary I. Bell, Scott Nivison, Michael M. Zavlanos, and Karl H. Johansson
\thanks{*This work is supported in part by AFOSR under award \#FA9550-19-1-0169
and by NSF under award CNS-1932011.}
\thanks{Zifan Wang and Karl H. Johansson are with Division of Decision and Control Systems, School of Electrical Enginnering and Computer Science, KTH Royal Institute of Technology, and also with Digital Futures, SE-10044 Stockholm, Sweden. Email: \{zifanw,kallej\}@kth.se.}
\thanks{Yi Shen and Michael M. Zavlanos are with the Department of Mechanical Engineering and Materials Science, Duke University, Durham, NC, USA. Email: \{yi.shen478, michael.zavlanos\}@duke.edu}%
\thanks{Zachary I. Bell and Scott Nivison are with the Air Force Research Laboratory, Eglin AFB, FL, USA. Email: \{scott.nivison, zachary.bell.10\}@us.af.mil.}%
}

\begin{document}

\maketitle
\thispagestyle{empty}
\pagestyle{empty}

\begin{abstract}
We consider risk-averse learning in repeated unknown games where the goal of the agents is to minimize their individual risk of incurring significantly high cost. Specifically, the agents use the conditional value at risk (CVaR) as a risk measure and rely on bandit feedback in the form
of the cost values of the selected actions at every episode to estimate their CVaR values and update their actions. A major challenge in using bandit feedback to estimate CVaR is that the agents can only access their own cost values, which, however, depend on the actions of all agents. To address this challenge, we propose a new risk-averse learning algorithm with momentum that utilizes the full historical information on the cost values. We show that this algorithm achieves sub-linear regret and matches the best known algorithms in the literature. We provide numerical experiments for a Cournot game that show that our method outperforms existing methods.

\end{abstract}

\section{Introduction}
In online convex games \cite{shalev2006convex,gordon2008no}, agents interact with each other in the same environment in order to minimize their individual cost functions. Many applications, including traffic routing \cite{sessa2019no} and online marketing \cite{shi2019no}, among many others \cite{hazan2019introduction}, can be modeled as online convex games. The individual cost functions depend on the joint decisions of all agents and are typically unknown but can be accessed via querying. Using this information, the agents can sequentially select the best actions that minimize their expected accumulated costs. 
This can be done using online learning algorithms whose performance is typically measured using the notions of regret \cite{hazan2019introduction} that quantify the gap between the agents' online decisions and the best decisions in hindsight. Many such online learning algorithms have been recently proposed, and have been shown to achieve sub-linear regret, which indicates that optimal decisions are eventually selected; see e.g., \cite{shalev2006convex,gordon2008no,shalev2011online}.
In this paper, here we focus on online convex games for high-stakes applications, where decisions that minimize the expected cost functions are not necessarily desirable. For example, in portfolio management, a selection of assets with the highest expected return is not necessarily desirable since it may have high volatility that can lead to catastrophic losses. To capture unexpected and catastrophic outcomes, risk-averse criteria have been widely proposed in place of the expectation, such as the Sharpe Ratio \cite{sharpe1994sharpe} and Conditional Value at Risk (CVaR) \cite{artzner1999coherent}. 

In this paper, we consider online convex games with risk-averse agents that employ CVaR as risk measure.
We assume that only cost function evaluations at selected decisions are accessible. 
Despite the practical utility of CVaR as a measure of risk in high-stakes applications, the theoretical analysis of risk-averse learning methods that employ the CVaR is limited, mainly since the CVaR values must be estimated from the distributions of the unknown cost functions. To avoid this problem, \cite{rockafellar2000optimization} proposes a variational definition of CVaR that allows to reformulate the computation of CVaR values into an optimization problem by introducing additional variables. 
Building on this reformulation, many risk-averse learning methods have been proposed, including \cite{soma2020statistical,cardoso2019risk,curi2019adaptive}, and their convergence rates have been analysed. Specifically, \cite{soma2020statistical} provides performance guarantees for stochastic gradient descent-type algorithms in risk-averse statistical learning problems; \cite{cardoso2019risk} addresses risk-averse online convex optimization problems with bandit feedback; and \cite{curi2019adaptive} proposes an adaptive sampling strategy for CVaR learning, and transforms the supervised learning problem into a zero-sum game between the sampler and the learner.
However, all the above works focus on single-agent learning problems. 

To the best of our knowledge, multi-agent risk-averse learning problems have not been analyzed in the literature, with only a few exceptions \cite{tamkin2019distributionally}, \cite{wang2022risk}. The work in \cite{tamkin2019distributionally} solves multi-armed bandit problems with finite actions, which are different from the convex games with continuous actions considered here.
Most closely related to this paper is our recent work in \cite{wang2022risk} proposing a new risk-averse learning algorithm for online convex games, where the agents utilize the cost evaluations at each time step to update their actions and achieve no-regret learning with high probability. 
To improve the learning rate of this algorithm, \cite{wang2022risk} also proposes two
variance reduction methods on the CVaR estimates and gradient estimates, respectively.
In this paper, we build on this work to significantly improve the performance of the algorithm using a notion of momentum that combines these variance reduction techniques. 
Specifically, the proposed new algorithm reuses samples of the cost functions from the full history of data while appropriately decreasing the weights of out-of-date samples.
The idea is that, using more samples, the unknown distribution of the cost functions needed to estimate the CVaR values can be better estimated. To do so, at each time step, the proposed  algorithm estimates the distribution of the cost functions using both
the immediate empirical distribution estimate and the distribution estimate from the previous time step, which summarizes all past cost values. 
Then, to construct the gradient estimates, more weight is allocated to more recent  information, similar to momentum gradient descent methods \cite{qian1999momentum}. 
To reduce the variance of the gradient estimate, we employ residual feedback \cite{zhang2020boosting}, which uses previous gradient estimates to update the current one.

Although \cite{wang2022risk} showed that reusing samples from the last iteration or residual feedback can individually improve the performance of the algorithm, it also identified that the combination of sample reuse and residual feedback is an unexplored and theoretically nontrivial problem. 
In this work, we combine sample reuse and residual feedback by proposing a novel zeroth-order momentum method that reuses all historical samples, compared to samples from only the last iteration. 
An additional important contribution of this work is that the proposed zeroth-order momentum method subsumes Algorithm 3 in \cite{wang2022risk} as a special case by appropriately choosing the momentum parameter. We show that our proposed method theoretically matches the best known result in \cite{wang2022risk} but empirically outperforms  other existing methods.

The rest of the paper is organized as follows. In Section~\ref{sec:problem}, we formally define the problem and provide some preliminary results. In Section~\ref{sec:preliminary}, we summarize results in \cite{wang2022risk} that are needed to develop our proposed method. In Section~\ref{sec:method}, we present our proposed method and analyze its regret. In Section~\ref{sec:simulation}, we numerically verify our method using a Cournot game example. Finally, in Section~\ref{sec:conclusion}, we conclude the paper.
\section{Problem definition}\label{sec:problem}
We consider a repeated game $\mathcal{G}$  with $N$ risk-averse agents. Each agent selects an action $x_i$ from the convex set $\mathcal{X}_i \subseteq \mathbb{R}^{d_i}$, and then receives a stochastic cost $J_i(x_i,x_{-i},\xi_i): \mathcal{X} \times \Xi_i \rightarrow \mathbb{R}$, where $x_{-i}$ represents all agents' actions except for agent $i$, and $\mathcal{X}=\Pi_{i=1}^N\mathcal{X}_i$ is the joint action space. 
We sometimes instead write $J_i(x,\xi_i)$ for ease of notation, where $x=(x_i,x_{-i})$ is the concatenated vector of all agents' actions. The cost $J_i(x_i,x_{-i},\xi_i)$ is stochastic, as $\xi_i \in \Xi_i$ is a stochastic variable.
We assume that the diameter of the convex set $\mathcal{X}_i$ is bounded by $D_x$ for all $i=1,\ldots,N$.
In what follows, we make the following assumptions on the class of games we consider in this paper.
\begin{assumption}\label{assump:J_convex}
The function $J_i(x_i,x_{-i},\xi_i)$ is convex in $x_i$ for every $\xi_i \in \Xi_i$ and bounded  by $U$, i.e., $|J_i(x,\xi_i)|\leq U$, for all $i=1,\ldots,N$.
\end{assumption}
\begin{assumption}\label{assump:J_Lips}
$J_i(x,\xi_i)$ is $L_0$-Lipschitz continuous in $x$ for every $\xi_i \in \Xi_i$, for all $i=1,\ldots,N$.
\end{assumption}
Assumptions \ref{assump:J_convex} and \ref{assump:J_Lips} are common assumptions in the analysis of online learning in games \cite{bravo2018bandit}.

The goal of the risk-averse agents is to minimize the risk of incurring significantly high costs, for possibly different risk levels. 
In this work, we utilize CVaR as the risk measure. For a given risk level $\alpha_i\in [0,1]$, CVaR is defined as the average of the $\alpha_i$ percent worst-case cost.
Specifically, we denote by $F_{i}^x(y)=\mathbb{P}\{J_i(x,\xi_i)\leq y \}$ as the cumulative distribution function (CDF) of the random cost $J_i(x,\xi_i)$ for agent $i$; and by  $J^{\alpha_i}$ the cost value at the $1-\alpha_i$ quantile of the distribution, also called the Value at Risk (VaR). 
Then, for a given risk level $\alpha_i\in [0,1]$, CVaR of the cost function $J_i(x,\xi_i)$ of agent $i$ is defined as \begin{align*} 
    C_i(x):& ={\rm{CVaR}}_{\alpha_i}[J_i(x,\xi_i)] \\
    :& = \mathbb{E}_F[J_i(x,\xi_i)|J_i(x,\xi_i)\geq J^{\alpha_i}].
\end{align*} 
Notice that the CVaR value is determined by the distribution function $F_i^x(y)$ for a given $\alpha_i$, so we sometimes write CVaR as a function of the distribution function, i.e.,  ${\rm{CVaR}}_{\alpha_i}[F_i^x]:={\rm{CVaR}}_{\alpha_i}[J_i(x,\xi_i)] $.
In addition, we assume that the agents cannot observe other agents' actions, but the only information they can observe is the cost evaluations of the jointly selected actions at each time step.

Given Assumptions 1 and 2, the following lemma lists properties of CVaR, that will be important in the analysis that follows. The proof can be found in \cite{cardoso2019risk}.
\begin{lemma}\label{lemma:cvar_property}
Given Assumptions \ref{assump:J_convex} and \ref{assump:J_Lips}, we have that $C_i(x_i,x_{-i})$ is convex in $x_i$ and $L_0$-Lipschitz continuous in $x$, for all $i=1,\ldots,N$. 
\end{lemma}

The following additional assumption on the variation of the CDFs is needed for the analysis that follows.
\begin{assumption}\label{assump_CDF_lips}
Let $F^w_i(y)= \mathbb{P}\{ J_i(w,\xi_i)\leq y\}$ and $F^v_i(y)= \mathbb{P}\{ J_i(v,\xi_i)\leq y\}$. There exist a constant $C_0>0$ such that 
\begin{align*}
    \mathop{\rm{sup}}_{y} |F^w_i(y)-F^v_i(y)| \leq C_0  \left\|w -v \right\|.
\end{align*}
\end{assumption}
Assumption \ref{assump_CDF_lips} states that the difference between two CDFs can be bounded by the difference between the corresponding two action profiles. It is less restrictive than Assumption~3 in \cite{wang2022risk} since it does not depend on the algorithm but only on the cost functions.

A common measure of the ability of the agents to learn online optimal decisions that minimize their individual risk-averse objective functions $C_i(x)$ is the algorithm regret, which is defined as the difference between the actual rewards and best rewards that the agent could have achieved by playing the single best action in hindsight.
Suppose the action sequences of agent~$i$ and the other agents in the team are $\{\hat{x}_{i,t} \}_{t=1}^T$ and $\{\hat{x}_{-i,t} \}_{t=1}^T$, respectively. Then, we define the regret of agent~$i$ as
\begin{align*}
    {\rm{R}}_{C_i}(T) = \sum_{t=1}^{T} C_i(\hat{x}_{i,t},\hat{x}_{-i,t}) - \mathop{\rm{min}}_{\tilde{x}_i \in \mathcal{X}_i} \sum_{t=1}^{T} C_i(\tilde{x}_i,\hat{x}_{-i,t}).
\end{align*}
An algorithm is said to be no-regret if its regret grows sub-linearly with the number of episodes $T$.
In this paper, we propose a no-regret learning algorithm, so that $\lim_{T\rightarrow \infty}\frac{{\rm{R}}_{C_i}(T)}{T} = 0$, $i=1,\ldots,N$.

\section{Preliminary Results}\label{sec:preliminary}
In this section, we summarize some results from \cite{wang2022risk} that lay the foundation for the subsequent analysis. Note that the CVaR values depend on the distributions of the cost functions, which are generally unknown. To estimate the distribution of the cost functions, and subsequently the CVaR values with only a few samples, a sampling strategy is proposed in \cite{wang2022risk} that uses a decreasing number of samples with the number of iterations. Specifically, at episode $t$, the sampling strategy used by the agents is designed as follows
\begin{align}\label{eq:sample_strategy}
    n_t = \lceil b  U^2 (T-t+1)^a \rceil,
\end{align}
where $\lceil \cdot \rceil$ is the regularized ceiling function, $T$ is the number of episodes, $U$ is the bound of $J_i$ as in Assumption \ref{assump:J_convex}, and $a,b \in(0,1)$ are parameters to be selected later. 

Using the cost evaluations at each episode, zeroth-order methods can be employed to estimate the CVaR gradients and then update the agents' actions. Specifically, at episode $t$, the agents calculate the number of samples $n_t$ according to \eqref{eq:sample_strategy}. Then, each agent perturbs the current action $x_{i,t}$ by an amount $\delta u_{i,t}$, where $u_{i,t} \in \mathbb{S}^{d_i}$ is a random perturbation direction sampled from the unit sphere $\mathbb{S}^{d_i}\subset \mathbb{R}^{d_i}$ and $\delta$ is the size of this perturbation. Next the agents play their perturbed actions $\hat{x}_{i,t}=x_{i,t}+\delta u_{i,t} $ for $n_t$ times, and obtain $n_t$ cost evaluations which are utilized to update their actions.
To facilitate the theoretical analysis, we define the $\delta$-smoothed function $C_i^{\delta}(x)= \mathbb{E}_{w_i \sim \mathbb{B}_i,u_{-i} \sim \mathbb{S}_{-i}}[C_i(x_i+\delta w_i,x_{-i}+\delta u_{-i})]$, where $\mathbb{S}_{-i}=\Pi_{j \neq i} \mathbb{S}_j$, and $\mathbb{B}_i$, $\mathbb{S}_i$ denote the unit ball and unit sphere in $\mathbb{R}^{d_i}$, respectively. The size of the perturbation $\delta$ here is related to a smoothing parameter that controls how well $C_i^{\delta}(x)$ approximates $C_i(x)$.
As shown in \cite{wang2022risk}, the function $C_i^{\delta}(x)$ satisfies the following properties.
\begin{lemma}\label{lemma:cdelta_property}
Let Assumptions \ref{assump:J_convex} and \ref{assump:J_Lips} hold. Then we have that
\begin{enumerate}
    \item $C_i^{\delta}(x_i,x_{-i})$ is convex in $x_i$,
    \item $C_i^{\delta}(x)$ is $L_0$-Lipschitz continuous in $x$,
    \item $|C_i^{\delta}(x)$-$C_i(x)|\leq \delta L_0 \sqrt{N}$,
    \item $\mathbb{E}[ \frac{d_i}{\delta} C_i(\hat{x}_t) u_{i,t}] = \nabla_i C_i^{\delta}(x_t)$.
\end{enumerate}
\end{lemma}
The last property in Lemma \ref{lemma:cdelta_property} shows that the term $\frac{d_i}{\delta} C_i(\hat{x}_t) u_{i,t}$ is an unbiased estimate of the gradient of the smoothed function $C_i^{\delta}(x)$. 
Next we present a lemma that helps bound the distance between two CVaR values by the distance between the two corresponding CDFs. The proof can be found in \cite{wang2022risk}.
\begin{lemma}\label{lemma:CVaR_bound}
Let $F$ and $G$ be two CDFs of two random variables and suppose the random variables are bounded by $U$. Then we have that
\begin{align*}
    |{\rm{CVaR}}_{\alpha}[F]-{\rm{CVaR}}_{\alpha}[G]| \leq \frac{U}{\alpha}  \mathop{\rm{sup}}_{y} |F(y)-G(y)|.
\end{align*}
\end{lemma}
\section{A zeroth-order momentum method}\label{sec:method}

In this section, we present a new one-point zeroth-order momentum method that combines sample reuse as in Algorithm 2 in \cite{wang2022risk} and residual feedback as in Algorithm 3 in \cite{wang2022risk}, but unlike Algorithm 2 in \cite{wang2022risk} that reuses samples only from the last iteration, it uses all past samples to update the agents' actions. The new algorithm is given as Algorithm 1.

\begin{algorithm}[t]
\caption{Risk-averse learning with momentum} \label{alg:algorithm1}
\begin{algorithmic}[1]
    \REQUIRE Initial value $x_0$, step size $\eta$, parameters $a$, $b$, $\delta$, $T$, risk level $\alpha_i$, $i=1,\cdots,N$.
    \FOR{$ {\rm{episode}} \; t=1,\ldots,T$}
        \STATE Select $n_t=\lceil b U^2 (T-t+1)^a\rceil$
        \STATE Each agent samples $u_{i,t} \in \mathbb{S}^{d_i}$,  $i=1,\ldots, N$
        \STATE Each agent play $\hat{x}_{i,t}=x_{i,t}+\delta u_{i,t} $, $i=1,\ldots, N$
        \FOR{$j=1,\ldots,n_t$}
        \STATE Let all agents play $\hat{x}_{i,t}$  
        \STATE Obtain $J_i(\hat{x}_{i,t},\hat{x}_{-i,t},\xi_{i}^j)$
        \ENDFOR
        \FOR{agent $ i=1,\ldots,N$}
        \STATE Build EDF $\bar{F}_{i,t}(y)$ 
        \STATE Estimate CVaR: $ {\rm{CVaR}}_{\alpha_i}[\bar{F}_{i,t}] $ 
        \STATE Construct gradient estimate\\ $\bar{g}_{i,t}=\frac{d_i}{\delta} \left( {\rm{CVaR}}_{\alpha_i} [\bar{F}_{i,t} ] -{\rm{CVaR}}_{\alpha_i} [\bar{F}_{i,t-1} ]\right) u_{i,t}$
        \STATE Update $x$: $x_{i,t+1} \leftarrow \mathcal{P}_{\mathcal{X}_i^{\delta}} ( x_{i,t} - \eta \bar{g}_{i,t})$
        \ENDFOR
    \ENDFOR
\end{algorithmic}
\end{algorithm}

Using the sampling strategy in \eqref{eq:sample_strategy}, each agent plays the perturbed action $\hat{x}_{i,t}$ for $n_t$ times and obtains $n_t$ samples at episode $t$.
For agent $i$, we denote the CDF of the random cost $J_i(\hat{x}_t,\xi_i)$ that is returned by the perturbed action $\hat{x}_t$ as $F_{i,t}^{\hat{x}_t}(y)= \mathbb{P}\{ J_i(\hat{x}_t,\xi_i) \leq y\}$. Since $\hat{x}_t$ depends on $t$, we write $F_{i,t}$ for $F_{i,t}^{\hat{x}_t}$ for ease of notation. 
Using bandit feedback in the form of finitely many cost evaluations, the agents cannot obtain the accurate CDF $F_{i,t}$. Instead, they construct an empirical distribution function (EDF) $\hat{F}_{i,t}$ of the cost $J_i(\hat{x}_t,\xi_i)$ using $n_t$ cost evaluations by
\begin{align}\label{eq:edf_alg1}
    \hat{F}_{i,t}(y)= \frac{1}{n_t}\sum_{j=1}^{n_t}\mathbf{1}\{ J_i(\hat{x}_t,\xi_i^j) \leq y\}.
\end{align}
To improve estimate of the CDF, we utilize past samples for all episodes $t\geq2$, and construct a modified distribution estimate $\bar{F}_{i,t}$ by adding a momentum term:
\begin{align}\label{eq:edf_bar}
    \bar{F}_{i,t}(y)= \beta \bar{F}_{i,t-1}(y) + (1-\beta) \hat{F}_{i,t}(y),
\end{align}
where $\beta\in[0,1)$ is the momentum parameter.
For $t=1$, we set $\bar{F}_{i,t}(y)= \hat{F}_{i,t}(y)$.
The agents utilize the distribution estimate $\bar{F}_{i,t}$ to calculate the CVaR estimates and further construct the gradient estimate using residual feedback as 
\begin{align}\label{eq:grad_alg1}
    \bar{g}_{i,t}=\frac{d_i}{\delta} \left( {\rm{CVaR}}_{\alpha_i} [\bar{F}_{i,t} ] -{\rm{CVaR}}_{\alpha_i} [\bar{F}_{i,t-1} ]\right) u_{i,t},
\end{align}
where $\delta$ is the size of the perturbation on the action $x_{i,t}$ defined above.
Depending on the size of the perturbation, the played action $\hat{x}_{i,t}$ may be infeasible, i.e., $\hat{x}_{i,t} \not\in \mathcal{X}_i$. To handle this issue, we define the projection set $\mathcal{X}_i^{\delta} =\{x_i \in \mathcal{X}_i | {\rm{dist}}(x_i,\partial \mathcal{X}_i) \geq \delta \}$ that we use in the projected gradient-descent update
\begin{align}\label{eq:projection}
    x_{i,t+1} = \mathcal{P}_{\mathcal{X}_i^{\delta}} ( x_{i,t} - \eta \bar{g}_{i,t}).
\end{align}

Note that the agents are not able to obtain accurate values of $C_i(\hat{x}_t)$ using finite samples. In fact, if we use the distribution estimate $\bar{F}_{i,t}$ in \eqref{eq:edf_bar} to calculate the CVaR values, there will be a CVaR estimation error, which is defined as 
\begin{align*}
    \bar{\varepsilon}_{i,t} := 
     {\rm{CVaR}}_{\alpha_i}[\bar{F}_{i,t} ]- {\rm{CVaR}}_{\alpha_i}[F_{i,t} ].
\end{align*} 
As a result, the gradient estimate $\bar{g}_{i,t}$ in \eqref{eq:grad_alg1} is biased since $\mathbb{E}[\bar{g}_{i,t}]=\mathbb{E}\left[ \frac{d_i}{\delta}(C_i(\hat{x}_t)+ \bar{\varepsilon}_{i,t})u_{i,t} \right]= \nabla_i C_i^{\delta}(x_t)+ \mathbb{E}\left[\frac{d_i}{\delta}\bar{\varepsilon}_{i,t}u_{i,t} \right]$, with the bias captured by the last term. 
In what follows, we present a lemma that bounds the sum of the CVaR estimation errors.
\begin{lemma}\label{lemma:sum_epsilon_bound}
Suppose that Assumption \ref{assump_CDF_lips} holds. Then,
the sum of the CVaR estimation errors satisfies
\begin{align}
    \sum_{t=1}^T \left\|\bar{\varepsilon}_{i,t}\right\|\leq \frac{U}{\alpha_i } \Big( \frac{\beta}{1-\beta} e_{i,1} + \frac{\beta\Sigma_1}{1-\beta}T  + \sum_{t=1}^T r_t \Big)
\end{align}
with probability at least $1-\gamma$, where $e_{i,t}:=\mathop{\rm{sup}}_{y} |\bar{F}_{i,t}(y)-F_{i,t}(y)|$, $r_t:=\sqrt{\frac{ {\rm{ln}}(2T/\gamma) }{2bU^2(T-t+1)^a}}$, and $\Sigma_1:=C_0 \eta \frac{d_i}{\delta}U\sqrt{N} + 2 C_0 \delta$.
\end{lemma}
\begin{proof}
To bound the CVaR estimation error, it suffices to bound the CDF differences using Lemma \ref{lemma:CVaR_bound}. Recall the definition of $\bar{F}_{i,t}$ in \eqref{eq:edf_bar}. Then, we have that 
\begin{align}\label{eq_term7}
    &\mathop{\rm{sup}}_{y} |\bar{F}_{i,t}(y)-F_{i,t}(y)| \nonumber \\
    =&  \mathop{\rm{sup}}_{y} |\beta \bar{F}_{i,t-1}(y) + (1-\beta) \hat{F}_{i,t}(y)- F_{i,t}(y)|\nonumber \\
    =& \mathop{\rm{sup}}_{y} | \beta (\bar{F}_{i,t-1}(y)- {F}_{i,t-1}(y)) + \beta ({F}_{i,t-1}(y)- {F}_{i,t}(y)) \nonumber\\
    &+ (1-\beta) (\hat{F}_{i,t}(y)-{F}_{i,t}(y)) |\nonumber\\
    \leq & \beta \mathop{\rm{sup}}_{y} |\bar{F}_{i,t-1}(y)- {F}_{i,t-1}(y)| + \beta \mathop{\rm{sup}}_{y} |{F}_{i,t-1}(y)- {F}_{i,t}(y)| \nonumber\\
     &+ (1-\beta) \mathop{\rm{sup}}_{y} |\hat{F}_{i,t}(y)-{F}_{i,t}(y)|.
\end{align}
We now bound the latter two terms in the right-hand-side of \eqref{eq_term7} separately. 
Using Assumption \ref{assump_CDF_lips} and Lemma \ref{lemma:cdelta_property}, we have that 
\begin{align}\label{eq_term8}
    &\mathop{\rm{sup}}_{y} |F_{i,t}(y)-F_{i,t-1}(y)| \nonumber \\
    \leq & C_0 \left\|\hat{x}_t - \hat{x}_{t-1} \right\| 
    \leq  C_0 \left\|x_t - x_{t-1} \right\| + 2 C_0 \delta \nonumber\\
    \leq & C_0 \eta \left\| \bar{g}_{t} \right\| + 2 C_0 \delta 
    \leq C_0 \eta \frac{d_i}{\delta}U\sqrt{N} + 2 C_0 \delta:=\Sigma_1,
\end{align}
where the last inequality holds since $\left\|\bar{g}_{t}\right\| =\sqrt{ \sum_{i=1}^N\left\|\bar{g}_{i,t}\right\|^2} \leq \frac{\sqrt{N} d_i U}{\delta}$.
Then, applying the Dvoretzky–Kiefer–Wolfowitz (DKW) inequality, we obtain that
\begin{align}\label{event_At}
    \mathbb{P}\left\{ \mathop{\rm{sup}}_{y} |F_{i,t}(y)-\hat{F}_{i,t}(y)| \geq \sqrt{\frac{ {\rm{ln}}(2/\bar{\gamma}) }{2n_t}}\right\} \leq \bar{\gamma}.
\end{align}
Define the events in (\ref{event_At}) as $A_t$ and denote $\gamma=\bar{\gamma}T$. Then,  the following inequality holds
\begin{align}\label{eq_term10}
\mathop{\rm{sup}}_{y} |F_{i,t}(y)-\hat{F}_{i,t}(y)| \leq \sqrt{\frac{ {\rm{ln}}(2T/\gamma) }{2n_t}}, \forall t=1,\ldots,T,
\end{align}
with probability at least $1-\gamma$ since $1-\mathbb{P} \{ \bigcup_{t=1}^T A_t\} \geq 1- \sum_{t=1}^T\mathbb{P}\{A_t\}\geq 1-T \frac{\gamma}{T}\geq1-\gamma$. 
Substituting the bounds in \eqref{eq_term8} and \eqref{eq_term10} into \eqref{eq_term7}, and using the definition of $e_{i,t}$, we get that
\begin{align}\label{eq:one_step_sum}
    e_{i,t}\leq \beta e_{i,t-1}+ \beta \Sigma_1 + (1-\beta)r_t.
\end{align}
Note that \eqref{eq:one_step_sum} holds for all $t\in\{2,\cdots,T\}$. By iteratively using this inequality, we have that 
\begin{align}\label{eq:eit}
    e_{i,t}\leq \frac{\beta}{1-\beta}\Sigma_1 + (1-\beta)\sum_{k=0}^{k-2}\beta^{k}r_{t-k}+ \beta^{t-1} e_{1,i}.
\end{align}
Taking the sum over $t=1,\ldots,T$ of both sides of \eqref{eq:eit}, we have that
\begin{align}
    \sum_{t=1}^T e_{i,t}\leq \frac{\beta}{1-\beta} e_{i,1} + \frac{\beta\Sigma_1}{1-\beta}T  + \sum_{t=1}^T r_t,
\end{align}
with probability at least $1-\gamma$.
Finally, using Lemma \ref{lemma:CVaR_bound}, we can obtain the desired result. 
\end{proof}

Lemma \ref{lemma:sum_epsilon_bound} bounds the CVaR estimation error caused by using past samples.
The next lemma quantifies the error due to residual feedback in the gradient estimate \eqref{eq:grad_alg1}. Specifically, it bounds the second moment of the gradient estimate.
\begin{lemma}\label{lemma:sum_gt}
Let Assumptions \ref{assump:J_convex}, \ref{assump:J_Lips} and \ref{assump_CDF_lips} hold. Suppose that the action is updated as in \eqref{eq:projection}. Then, the second moment of the gradient estimate satisfies
\begin{align}\label{eq:gt}
    \sum_{t=1}^T \left\|\bar{g}_{t}\right\|^2 
    \leq \frac{1}{1-\sigma} \left\| \bar{g}_{1}\right\|^2 + \frac{1}{1-\sigma}\Sigma_2 T,
\end{align}
where $\bar{g}_t:=(\bar{g}_{i,t}, \bar{g}_{-i,t})$ is the concatenated vector of all agents' gradient estimates,  $\sigma=\frac{4d_i^2L_0^2N \eta^2}{\delta^2}$, $\Sigma_2=16d_i^2N^2L_0^2 + \frac{48C_0^2 d_i^4 U^4 N \eta^2 \beta^2}{\delta^4 (1-\beta)^2} (\sum_i \frac{1}{\alpha_i^2}) + \frac{192 C_0^2 d_i^2 U^2 \beta^2}{ (1-\beta)^2}(\sum_i \frac{1}{\alpha_i^2}) + \frac{12 \ln(2T/\gamma) d_i^2}{b\delta^2}(\sum_i \frac{1}{\alpha_i^2}) + \frac{24 d_i^2 U^2 \beta^2 e_m^2}{\delta^2}(\sum_i \frac{1}{\alpha_i^2}) $, and
$e_m:= \mathop{\max}_i \{e_{i,1}\}$.
\end{lemma}

\begin{proof}
Using the fact that $r_t\leq \sqrt{\frac{ {\rm{ln}}(2T/\gamma) }{2bU^2}}:=r$ and  \eqref{eq:eit}, we have that
\begin{align}\label{eq:eit_ineq}
    \left\| e_{i,t} \right\|^2 & \leq \left\| \frac{\beta}{1-\beta}\Sigma_1 + r+ \beta e_{m} \right\|^2  \nonumber\\
    & \leq 3\left\| \frac{\beta}{1-\beta}\Sigma_1 \right\|^2 + 3\left\| r\right\|^2  + 3\left\| \beta e_{m}\right\|^2,
\end{align}
where the last inequality holds since $\left\| a+b+c\right\|^2 \leq 3 \left\|a\right\|^2+3 \left\|b\right\|^2+3 \left\|c\right\|^2$. Then, applying Lemma \ref{lemma:CVaR_bound} to \eqref{eq:eit_ineq}, we obtain that
\begin{align}\label{eq:epsibar_bound}
    \left\|\bar{\varepsilon}_{i,t}\right\|^2 \leq \frac{U^2}{\alpha_i^2}\left( \frac{3\beta^2  \Sigma_1^2}{(1-\beta)^2} + 3 \left\| r\right\|^2+ 3\beta^2 \left\| e_m\right\|^2\right).
\end{align}
By the definition of $\bar{g}_{i,t}$ in \eqref{eq:grad_alg1}, we have that
\begin{align}    \label{eq_bar_git}
    &\left\|\bar{g}_{i,t}\right\|^2  
    =\frac{d_i^2}{\delta^2} \left( ({\rm{CVaR}}_{\alpha_i} [\bar{F}_{i,t} ]- {\rm{CVaR}}_{\alpha_i} [\bar{F}_{i,t-1} ] ) u_{i,t}\right)^2 \nonumber \\
    &\leq \frac{d_i^2}{\delta^2} \left({\rm{CVaR}}_{\alpha_i} [\bar{F}_{i,t} ]- {\rm{CVaR}}_{\alpha_i} [\bar{F}_{i,t-1} ] \right)^2 \left\| u_{i,t}\right\|^2 \nonumber \\
    &\leq \frac{d_i^2}{\delta^2} \Big( 2 ({\rm{CVaR}}_{\alpha_i} [F_{i,t} ]- {\rm{CVaR}}_{\alpha_i} [F_{i,t-1} ])^2  \nonumber \\
    &\quad + 2 (\bar{\varepsilon}_{i,t}-\bar{\varepsilon}_{i,t-1})^2 \Big)  \left\| u_{i,t}\right\|^2 \nonumber \\
    & \leq \frac{d_i^2}{\delta^2} \Big( 2 L_0^2 \left\|\hat{x}_t-\hat{x}_{t-1} \right\|^2 + 4|\bar{\varepsilon}_{i,t} |^2 + 4|\bar{\varepsilon}_{i,t-1} |^2 \Big),
\end{align}
where the last inequality is due to the fact that the CVaR function is Lipschitz continuous as shown in Lemma \ref{lemma:cvar_property}. We now bound the term $\left\|\hat{x}_t-\hat{x}_{t-1} \right\|^2$ in \eqref{eq_bar_git}. Recalling the update rule $x_{i,t}= \mathcal{P}_{\mathcal{X}_i^{\delta}} ( x_{i,t-1} - \eta_{i} \bar{g}_{i,t-1})$, we have that 
\begin{align*}
    &\left\|\hat{x}_t-\hat{x}_{t-1} \right\|^2  
    =  \left\| x_t-x_{t-1} +\delta u_t-\delta u_{t-1} \right\|^2  \nonumber \\
    \leq& 2 \left\| x_t - x_{t-1}\right\|^2 +2\delta^2(2\left\|u_t  \right\|^2 + 2\left\|u_{t-1} \right\|^2 )  \nonumber \\
    \leq& 2  \left\| x_t - x_{t-1}\right\|^2 + 8 N\delta^2
    \leq 2 \eta^2 \left\|\bar{g}_{t-1} \right\|^2 + 8 N\delta^2. \nonumber 
\end{align*}
Substituting the above inequality and the inequality in \eqref{eq:epsibar_bound} into the right hand side of \eqref{eq_bar_git} and summing both sides over all agents $i=1,\ldots,N$, we have that 
\begin{align}\label{eq_sum_of_g}
    &\left\|\bar{g}_{t}\right\|^2 =\sum_{i=1}^N \left\|\bar{g}_{i,t}\right\|^2 \nonumber\\
    &\leq \frac{d_i^2 N }{\delta^2} \Big( 4L_0^2 \eta^2 \left\|\bar{g}_{t-1} \right\|^2 + 16 L_0^2N\delta^2 \Big) \nonumber \\
    &\quad + \frac{8 d_i^2 }{\delta^2} \sum_{i=1}^N \frac{U^2}{\alpha_i^2}\left( \frac{3\beta^2  \Sigma_1^2}{(1-\beta)^2} + 3 \left\| r\right\|^2+ 3\beta^2 \left\| e_m\right\|^2\right) \nonumber \\
    &\leq \sigma \left\| \bar{g}_{t-1}\right\|^2 + 16d_i^2N^2L_0^2+ \frac{12 \ln(2T/\gamma) d_i^2}{b\delta^2} (\sum_i \frac{1}{\alpha_i^2}) \nonumber \\
    & \quad + \frac{24d_i^2 U^2\beta^2  }{\delta^2 (1-\beta)^2} (C_0 \eta \frac{d_i}{\delta}U\sqrt{N} + 2 C_0 \delta)^2 (\sum_i \frac{1}{\alpha_i^2}) \nonumber \\
    &\quad +\frac{24 d_i^2 U^2 \beta^2 e_m^2}{\delta^2}(\sum_i \frac{1}{\alpha_i^2}) \nonumber \\
    & \leq \sigma \left\| \bar{g}_{t-1}\right\|^2 + \Sigma_2,
\end{align}
where the last inequality is due to the fact that $(C_0 \eta \frac{d_i}{\delta}U\sqrt{N} + 2 C_0 \delta)^2\leq 2C_0^2\eta^2\frac{d_i^2}{\delta^2}U^2 N + 8C_0^2\delta^2$. Summing \eqref{eq_sum_of_g} over all episode $t=1,\ldots,T$ completes the proof. 
\end{proof}

Before proving the main theorem, we present a regret decomposition lemma that links the regret to the errors bounds on the CVaR estimates and gradient estimates as provided in Lemmas \ref{lemma:sum_epsilon_bound} and \ref{lemma:sum_gt}.
\begin{lemma}\label{lemma:regret_decomp}
Let Assumptions \ref{assump:J_convex}, \ref{assump:J_Lips} and \ref{assump_CDF_lips}  hold. Then, the regret of Algorithm \ref{alg:algorithm1} satisfies 
\begin{align}\label{eq:regret_decomp}
    {\rm{R}}_{C_i}(T) \leq &\frac{D_x^2}{2\eta}  +\frac{ \eta}{2}\mathbb{E} \big[ \sum_{t=1}^{T} \left\|\bar{g}_{i,t}\right\|^2 \big]   + (4 \sqrt{N}+\Omega)L_0\delta T \nonumber \\ 
    &+ \frac{d_i D_x }{\delta}\sum_{t=1}^T|\bar{\varepsilon}_{i,t}|,
\end{align}
where $\Omega>0$ is a constant that represents the error from projection $\mathcal{P}$.
\end{lemma}
\begin{proof}
The proof can be adapted from Lemma 5 in \cite{wang2022risk} and is omitted due to the space limitations. 
\end{proof}
\begin{theorem}\label{thm:1}
Let Assumptions \ref{assump:J_convex}, \ref{assump:J_Lips} and \ref{assump_CDF_lips}  hold and select $\eta=\frac{D_x  }{d_i L_0 N}T^{-\frac{3a}{4}}$,  $\delta=\frac{D_x }{ N^{\frac{1}{6}}}T^{-\frac{a}{4}}$, $\beta=\frac{1}{ U^2  T^{\frac{a}{4}} } $. 
Suppose that $n_t$ is chosen as in \eqref{eq:sample_strategy} with $a\in(0,1)$, and the EDF and the gradient estimate are defined as in \eqref{eq:edf_alg1} and (\ref{eq:grad_alg1}), respectively. Then, when $T\geq (8N^{\frac{2}{3}} )^{\frac{1}{a}}$, Algorithm \ref{alg:algorithm1} achieves regret
$$
{\rm{R}}_{C_i}(T)= \mathcal{O}( D_x d_i L_0 NS(\alpha) \ln(T/\gamma) T^{1-\frac{a}{4}} ),
$$
with probability at least $1-\gamma$ and where $S(\alpha):=\sum_{i=1}^N\frac{1}{\alpha_i^2}$.
\end{theorem}
\begin{proof}
Substituting the bounds in Lemmas \ref{lemma:sum_epsilon_bound} and \ref{lemma:sum_gt} into Lemma \ref{lemma:regret_decomp}, we have that
\begin{align}
    {\rm{R}}_{C_i}(T) \leq &\frac{D_x^2}{2\eta}  +\frac{ \eta}{2} \left( \frac{1}{1-\sigma} \left\| \bar{g}_{1}\right\|^2 + \frac{1}{1-\sigma}\Sigma_2 T \right)
      \nonumber \\ 
    &+ \frac{d_i D_x U}{\delta \alpha_i} \Big( \frac{\beta}{1-\beta} e_{i,1} + \frac{\beta\Sigma_1}{1-\beta}T  + \sum_{t=1}^T r_t \Big) \nonumber \\
    & + (4 \sqrt{N}+\Omega)L_0\delta T.
\end{align}
Substituting $\delta$, $\eta$ and $\beta$ into the above inequality, we can obtain the desired result. The detailed proof is straightforward and is omitted due to the space limitations. 
\end{proof}
Note that Algorithm \ref{alg:algorithm1} in fact achieves the regret  
${\rm{R}}_{C_i}(T) =\tilde{\mathcal{O}} (T^{1-\frac{a}{4}})$, in which we use the notation $\tilde{\mathcal{O}}$ to hide constant factors and poly-logarithmic factors of $T$. This result matches the best result in \cite{wang2022risk}. Compared to \cite{wang2022risk}, the analysis of Algorithm 1 that combines historical information and residual feedback is nontrivial and requires an appropriate selection of the momentum parameter to show convergence. Note that Algorithm 3 in \cite{wang2022risk} is a special case of Algorithm 1 for the momentum parameter $\beta=0$.

\section{Numerical Experiments}\label{sec:simulation}
In this section, we illustrate the proposed algorithm on a Cournot game. Specifically, we consider two risk-averse agents $i=1,2$ that compete with each other. Each agent determines the production level $x_i$ and receives an individual cost feedback which is given by $J_i=1-(2-\sum_jx_j)x_i+ 0.2x_i+\xi_i x_i$, where $\xi_i\sim U(0,1)$ is a uniform random variable. Here we utilize the term $\xi_i x_i$ to represent the uncertainty occurred in the market, which is proportional to the production level $x_i$. The agents have their own risk levels $\alpha_i$ and they aim to minimize the risk of incurring high costs, i.e., the CVaR of their cost functions. 

Note that the regret for agent $i$ depends on the sequence of the other agent's actions $\{x_{-i,t}\}_{t=1}^T$, and this sequence depends on the algorithm. Therefore, it is not appropriate to compare different algorithms in terms of the regret. 
Instead, we use the empirical performance to evaluate the algorithms. Specifically, we compute the CVaR values at each episode, and observe how fast the agents can minimize the CVaR values during the learning process.

To compute the distribution function $\bar{F}_{i,t}$, we divide the interval $[0,U]$ into $n$ bins of equal width, and we approximate the expectation by the sum of finite terms.
We compare our zeroth-order momentum method with the two algorithms in \cite{wang2022risk}, which we term as the algorithm with sample reuse and the algorithm with residual feedback, respectively. 
The number of samples $n_t$ at each episode for Algorithm~1 is shown in Figure~\ref{fig_samples}. We set the momentum parameter $\beta=0.5$.
Each algorithm is run for 20 trials and the parameters of these algorithms are separately optimally tuned. The CVaR values of these algorithms are presented in Figure \ref{fig_momen}. We observe that all the algorithms finally converge to the same CVaR values, but the zeroth-order momentum method proposed here converges faster. This is because our method uses all past samples to get a better estimate of the CVaR values and thus allows for a larger step size. Moreover, we also observe that the standard deviation of both our proposed algorithm and the algorithm with residual feedback is sufficiently small, which verifies the variance reduction effect of using residual feedback.

\begin{figure}[t] 
\begin{center}
\centerline{\includegraphics[width=0.9\columnwidth]{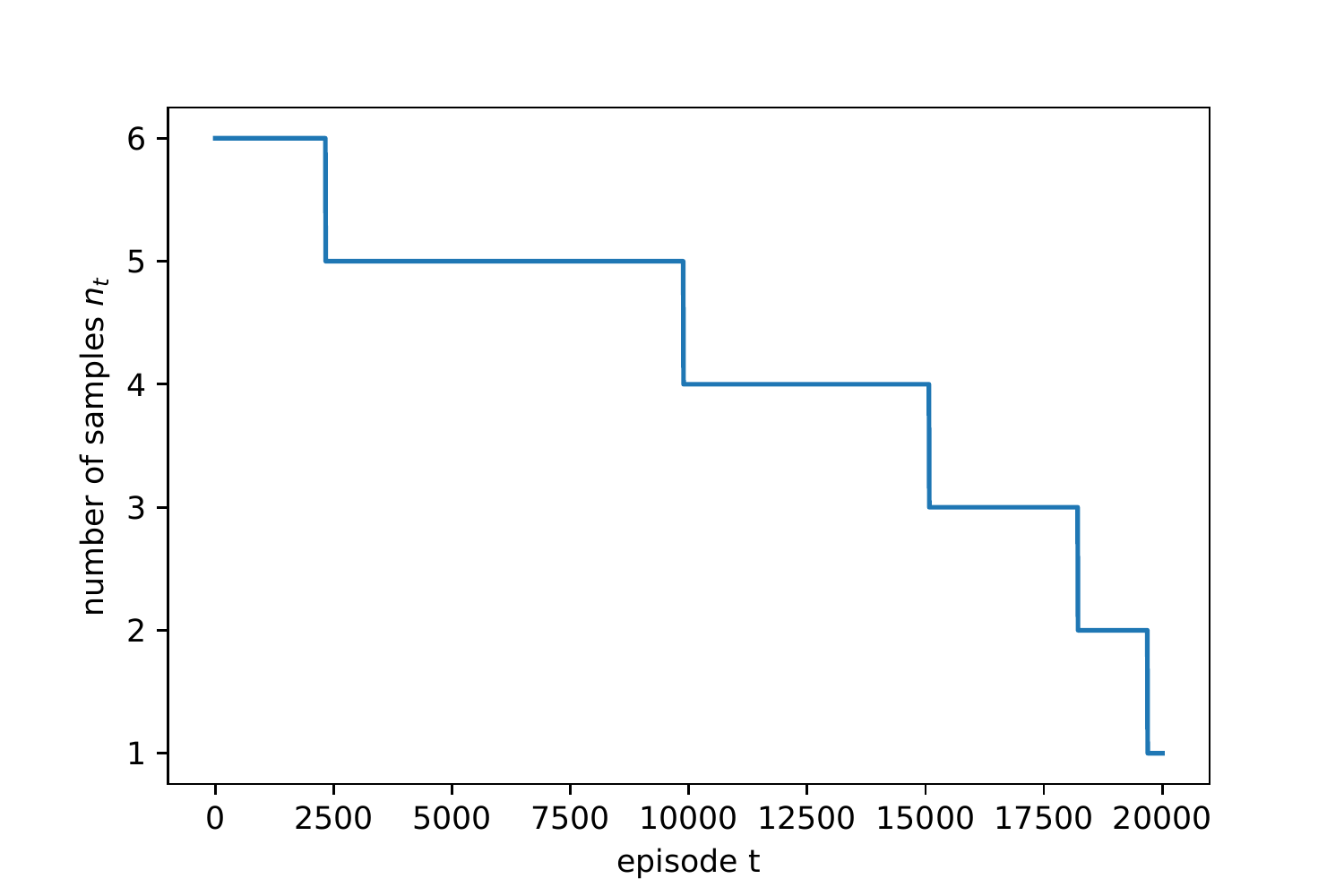}}
\caption{The number of samples of Algorithm \ref{alg:algorithm1}.}
\label{fig_samples}
\end{center}
\vskip -0.2in
\end{figure}

\begin{figure}[t]
\begin{center}
\centerline{\includegraphics[width=0.9\columnwidth]{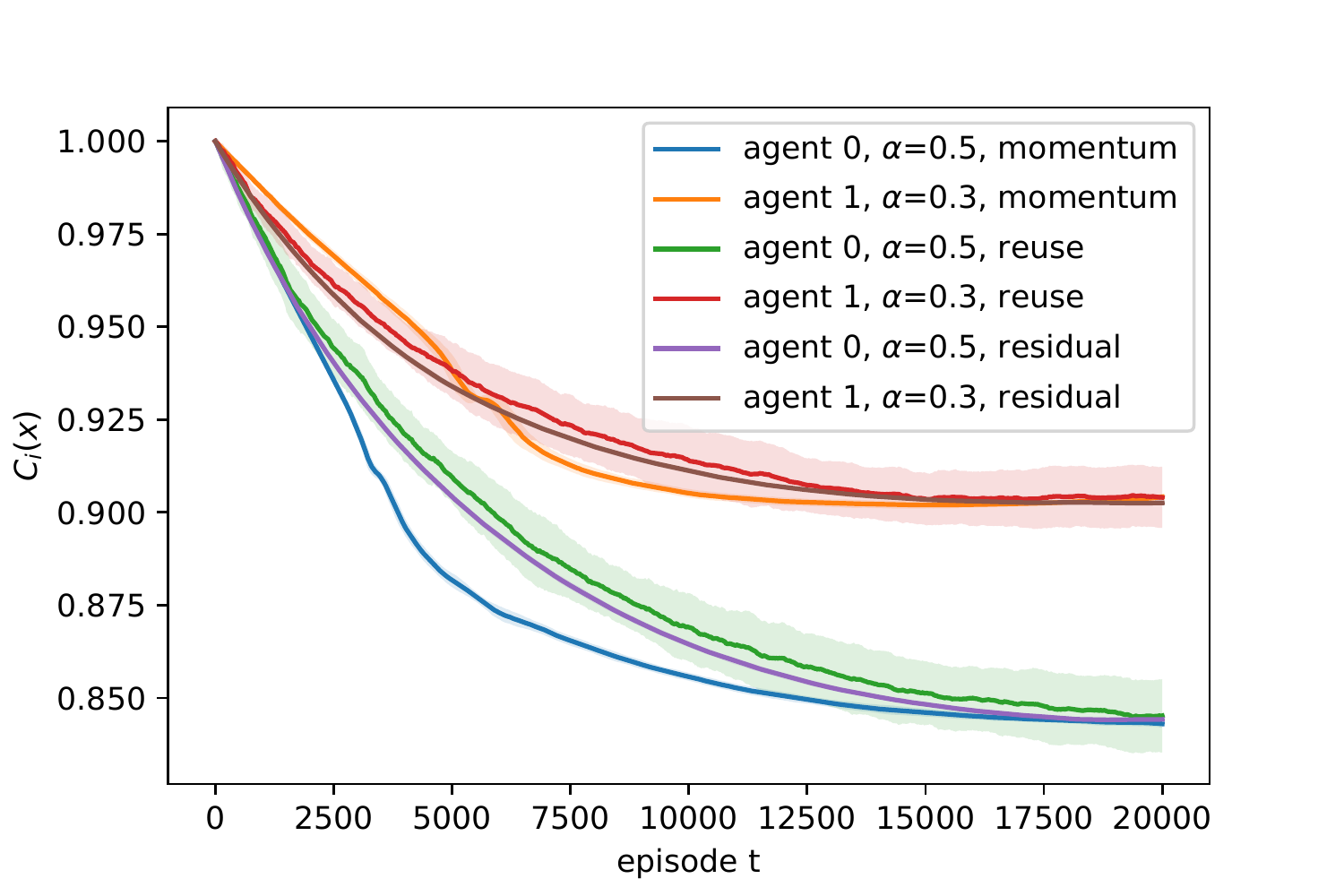}}
\caption{Comparative results among our proposed zeroth-order momentum method, and the three algorithms in \cite{wang2022risk}. Shaded areas represent $\pm$ one standard deviation over 20 runs.}
\label{fig_momen}
\end{center}
\vskip -0.2in
\end{figure}

\begin{figure}[t]
\begin{center}
\centerline{\includegraphics[width=0.9\columnwidth]{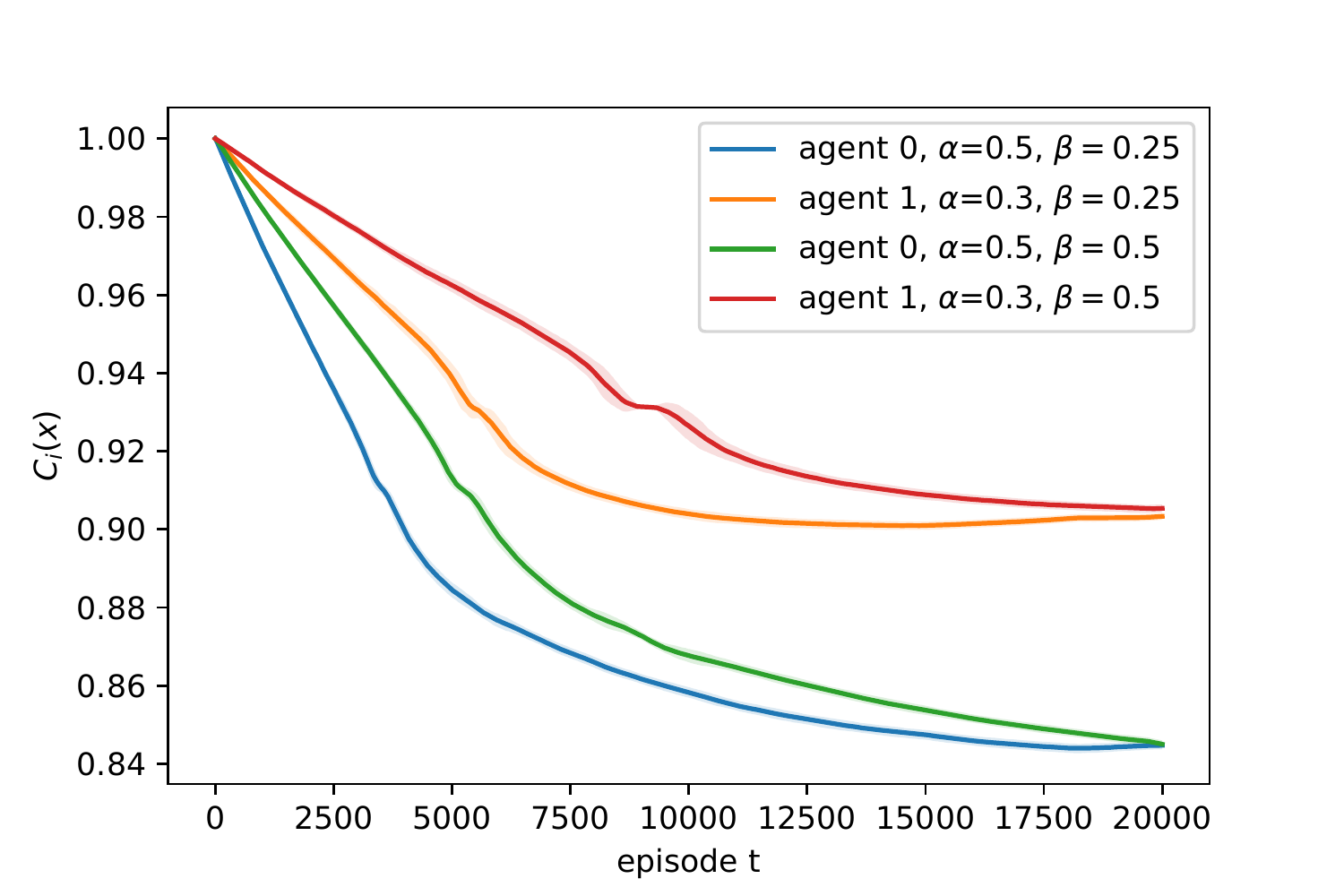}}
\caption{CVaR values achieved by Algorithm \ref{alg:algorithm1} with different $\beta$ values. Shaded areas represent $\pm$ one standard deviation over 20 runs.}
\label{fig_vari_beta}
\end{center}
\vskip -0.2in
\end{figure}

Moreover, we explore the effect of different momentum parameters $\beta$. In Figure \ref{fig_vari_beta}, we present the achieved CVaR values of our method for various $\beta$ values and all other parameters unchanged. Recalling the definition of the distribution estimate in \eqref{eq:edf_bar}, large values of $\beta$ place more weight on outdated cost values and thus reduce the convergence speed.

\section{Conclusion}\label{sec:conclusion}
In this work, we proposed a zeroth-order momentum method for online convex games with risk-averse agents. 
The use of momentum that employs past samples allowed to improve the ability of the algorithm to estimate the CVaR values. Zeroth-order estimation of the CVaR gradients using residual feedback allowed us to reduce the variance of the CVaR gradient estimates. 
We showed that the proposed algorithm theoretically achieves no-regret learning with high probability, matching the best known result in literature.
Moreover, we provided numerical simulations that demonstrated the superior performance of our method in practice. While sample reuse and residual feedback had been both separately shown to improve the performance of online learning, here we showed that their combination yields yet better performance.

\bibliography{ref}
\bibliographystyle{IEEEtran}

\end{document}